\documentclass{article} 
\usepackage{iclr2020_conference,times}

\usepackage{subfig}

\usepackage[utf8]{inputenc} 
\usepackage[T1]{fontenc}    
\usepackage{url}            
\usepackage{booktabs}       
\usepackage{amsfonts}       
\usepackage{nicefrac}       
\usepackage{microtype}      
\usepackage[utf8]{inputenc}
\usepackage{amsmath,amssymb,graphicx,url,amsthm}
\usepackage{times}
\usepackage{wrapfig,amsfonts,bm,color,enumitem,algorithm}
\usepackage{amsmath,amsfonts,amssymb,amsthm, color}
\usepackage{algorithm,algorithmicx}
\usepackage{epsfig,wrapfig,epsf,subfig}
\usepackage[outercaption]{sidecap} 
\usepackage{times}
\usepackage{fancyvrb}
\usepackage{enumerate}
\usepackage{tikz}
\usepackage{bm}
\usetikzlibrary{arrows,shapes}
\usetikzlibrary{patterns}
\usepackage{graphicx}
\usepackage{comment}
\usepackage{ltxtable}
\usepackage{nicefrac}
\usepackage{multirow}
\usepackage{setspace}
\usepackage[colorlinks=true, linkcolor=black, bookmarks=true]{hyperref}
\usepackage[noend]{algpseudocode}
\usepackage[utf8]{inputenc} 
\usepackage[T1]{fontenc}    
\usepackage{url}            
\usepackage{booktabs}       
\usepackage{microtype}      
\usepackage{pbox}
\usepackage{xfrac}

\makeatletter
\newtheorem*{rep@theorem}{\rep@title}
\newcommand{\newreptheorem}[2]{%
\newenvironment{rep#1}[1]{%
 \def\rep@title{#2 \ref{##1}}%
 \begin{rep@theorem}}%
 {\end{rep@theorem}}}
\makeatother

\newtheorem{theorem}{Theorem}[section]
\newtheorem{lemma}[theorem]{Lemma}

\newtheorem{definition}[theorem]{Definition}
\newreptheorem{theorem}{Theorem}
\newreptheorem{lemma}{Lemma}
\newreptheorem{definition}{Definition}

\theoremstyle{definition}

\newcounter{saveenumi}

\algrenewcommand\algorithmicrequire{\textbf{Parameters:}}
\algrenewcommand\algorithmicensure{\textbf{Initialization:}}

\usepackage{nicefrac}
\usepackage{natbib}
\usepackage{graphicx}

\newcommand{\plcomment}[1]{}

\usepackage{hyperref}
\usepackage{url}

\newcommand{\cB}{\mathcal{B}}

\newcommand{\cF}{\mathcal{F}}

\newcommand{\cK}{\mathcal{K}}

\newcommand{\cM}{\mathcal{M}}

\newcommand{\cN}{\mathcal{N}}
\newcommand{\calO}{\mathcal{O}}

\newcommand{\mathE}{\mathbb{E}}

\newcommand{\hy}{\hat{y}}

\newcommand{\tJ}{\tilde{J}}
\newcommand{\tildeK}{\tilde{K}}
\newcommand{\tK}{\tilde{K}}

\newcommand{\tV}{\tilde{V}}
\newcommand{\tTheta}{\tilde{\Theta}}

\DeclareMathOperator{\VEC}{vec}
\DeclareMathOperator{\vecrm}{vec}
\DeclareMathOperator{\op}{op}

\newcommand{\removed}[1]{}

\newcommand{\norm}[1]{\left\lVert{#1}\right\rVert}

\newcommand{\N}{\mathbb{N}}

\newcommand{\R}{\mathbb{R}}

\newcommand{\E}{\mathbb{E}}

\renewcommand{\E}{\mathop\mathbb{E}}

\newcommand{\rhomax}{\rho_{\mathrm{max}}}

\newcommand*{\eqdef}{\stackrel{\textup{def}}{=}}

\title{Generalization bounds \\
       for deep convolutional neural networks \\}

\author{
\hspace*{-5pt}
Philip M. Long\thanks{Authors are ordered alphabetically.}
\hspace*{4pt}
and Hanie Sedghi\footnotemark[1] \\
Google Brain \\
\texttt{\{plong,hsedghi\}@google.com} \\
}

\iclrfinalcopy 
\begin{document}

\maketitle

\begin{abstract}
We prove bounds on the generalization error of convolutional networks.
The bounds are in terms of the training loss, the number of
parameters, the Lipschitz constant of the loss and the distance from
the weights to the initial weights.  They are independent of the
number of pixels in the input, and the height and width of hidden
feature maps.
We present experiments using CIFAR-10 with varying
hyperparameters of a deep convolutional network, comparing our bounds
with practical generalization gaps.
\end{abstract}

\section{Introduction}

Recently, substantial progress has been made regarding theoretical
analysis of the generalization of deep learning models
\citep[see][]{neyshabur2015norm,zhang2016understanding,dziugaite2017computing,bartlett2017spectrally,neyshabur2017exploring,neyshabur2018pac,arora2018stronger,golowich2018size,neyshabur2019towards,wei2019data,cao2019generalization,daniely2019generalization}.
One interesting point that has been explored, with roots in
\citep{bartlett1998sample}, is that even if there are many parameters,
the set of 
models that can be represented
using weights with small magnitude is
limited enough to provide leverage for induction
\citep{neyshabur2015norm,bartlett2017spectrally,neyshabur2018pac}.
Intuitively, if the weights start small, since the most popular
training algorithms make small, incremental updates that get smaller
as the training accuracy improves, there is a tendency for these
algorithms to produce small weights.  (For some deeper theoretical
exploration of implicit bias in deep learning and related settings,
see
\citep{gunasekar2017implicit,gunasekar2018characterizing,gunasekar2018implicit,DBLP:conf/icml/MaWCC18}.)
Even more recently, authors have proved generalization bounds in terms
of the distance from the initial setting of the weights instead of the
size of the weights
\citep{dziugaite2017computing,bartlett2017spectrally,neyshabur2019towards,nagarajan2019generalization}.
This is important because small initial weights may promote vanishing
gradients; it is advisable instead to choose initial weights that
maintain a strong but non-exploding signal as computation flows
through the network
\citep[see][]{lecun2012efficient,glorot2010understanding,saxe2013exact,he2015delving}.
A number of recent theoretical analyses have shown that, for a large
network initialized in this way, 
accurate models can be found
by traveling a short distance
in parameter space
\citep[see][]{du2019gradient,du2019deep,allen2019convergence,zou2018stochastic,lee2019wide}.
Thus, the distance from initialization may be expected to be
significantly smaller than the magnitude of the weights.  Furthermore,
there is theoretical reason to expect that, as the number of
parameters increases, the distance from initialization decreases.
This motivates generalization bounds in terms of distance from
initialization \citep{dziugaite2017computing,bartlett2017spectrally}.

Convolutional layers are used in all competitive deep neural
network architectures applied to image processing tasks.  
The most influential 
generalization analyses in terms of distance from initialization
have thus far concentrated on networks with fully connected
layers.  Since a convolutional layer has an alternative representation
as a fully connected layer, these analyses apply in the
case of convolutional networks, but, intuitively, the weight-tying
employed in the convolutional layer constrains the set of
functions computed by the layer.  This additional restriction
should be expected to aid generalization.  

In this paper, we prove new generalization bounds for convolutional
networks that take account of this effect.  
As in earlier analyses for
the fully connected case, our bounds are in terms of the distance from
the initial weights, and the number of parameters.  
Additionally, our bounds 
independent of the number of pixels in the input, or the height and
width of the hidden feature maps.

Our most general bounds apply to networks including both convolutional
and fully connected layers, and, as such, they also apply for purely
fully connected networks.  In contrast with earlier bounds
for settings like the one considered here, our
bounds are in terms of a sum over layers of the distance from
initialization of the layer.  Earlier bounds
were in terms of product of these distances
which led to an
exponential dependency on depth. Our bounds have linear dependency
on depth which is more aligned with practical observations.


As is often the case for generalization analyses, the central
technical lemmas are bounds on covering numbers.
Borrowing a technique due to
\citet{barron1999risk}, these are proved by
bounding the Lipschitz constant of the mapping from the parameters to
the loss of the functions computed by the networks.  
(Our proof also borrows ideas from
the analysis of the fully connected case, especially 
\citep{bartlett2017spectrally,neyshabur2018pac}.)
Covering bounds may be applied to obtain a huge variety
of generalization bounds.  We present two examples for
each covering bound.
One is a standard bound
on the difference between training and test error.  
Perhaps the more relevant bound 
has the flavor of ``relative error''; it is especially
strong when the training loss is small, as is often the
case in modern practice.  Our covering bounds are polynomial in the
inverse of the granularity of the cover.  
Such bounds seem to be especially useful for bounding the relative error.

In particular, our covering bounds are of the form $(B/\epsilon)^W$,
where $\epsilon$ is the granularity of the cover, $B$ is 
proportional to the Lipschitz
constant of a mapping from parameters to functions, and $W$ is the
number of parameters in the model.  We apply a bound from the
empirical process literature in terms of covering bounds of this form
due to \citet{gine2001consistency}, who paid particular attention to
the dependence of estimation error on $B$.
This bound may be helpful
for other analyses of the generalization of deep learning in terms of
different notions of distance from initialization.
(Applying bounds in terms of Dudley's entropy integral in
the standard way leads to an exponentially worse dependence
on $B$.)

{\bf Related previous work.}  \citet{du2018many} proved 
bounds for
CNNs in terms of the number of parameters, for two-layer networks.
\citet{arora2018stronger} analyzed the generalization of networks
output by a compression scheme applied to CNNs.
\citet{DBLP:conf/icml/ZhouF18} provided a generalization guarantee for
CNNs satisfying a constraint on the rank of matrices formed from their
kernels.  \citet{li2018tighter} analyzed the generalization of CNNs
under other constraints on the parameters.  \citet{lee2018learning}
provided a size-free bound for CNNs in a general unsupervised learning
framework that includes PCA and codebook learning.

{\bf Related independent work.}  
\citet{ledent2019cnngen} proved bounds for CNNs 
that also took account of the effect of weight-tying. 
(Their bounds retain the exponential dependence on the
depth of the network from earlier work, and are otherwise qualitatively
dissimilar to ours.)
\citet{wei2019improved} obtained bounds for fully connected networks
with an improved dependence on the depth of the network.
\citet{daniely2019generalization} 
obtained improved bounds for constant-depth fully-connected networks.
\citet{jiang2019fantastic} conducted a wide-ranging empirical study
of the dependence of the generalization gap on a variety of
quantities, including distance from initialization.

{\bf Notation.}
If $K^{(i)}$ is the kernel of convolutional layer number $i$, then
$\op(K^{(i)})$ refers to its operator matrix~\footnote{Convolution is a linear operator and can thus be written as a matrix-vector product. The operator matrix of kernel $K$, refers to the matrix that describes convolving the input with kernel $K$. For details, see~\citep{sedghi2018singular}. } 
and
$\VEC(K^{(i)})$ denotes the vectorization of the kernel tensor $K^{(i)}$.
For matrix $M$, $\Vert M \Vert_2$ denotes the operator norm of $M$. 
%
For vectors, $|| \cdot ||$ represents the Euclidian norm,
and $|| \cdot ||_1$ is the $L_1$ norm.
For a multiset $S$ of elements of some set $Z$, and
a function $g$ from $Z$ to $\R$, let
$\E_S[g] = \frac{1}{m} \sum_{t=1}^m g(z_t)$.
We will denote the function
parameterized by $\Theta$ by $f_{\Theta}$.

\section{Bounds for a basic setting}

In this section, we provide a bound for a clean and simple
setting.

\subsection{The setting and the bounds}
\label{s:setting.basic}

In the basic setting, the input and all hidden
layers have the same number $c$ of channels.  Each input
$x \in \mathbb{R}^{d \times d \times c}$
satisfies $\Vert \VEC(x) \Vert \leq 1$.

We consider a deep convolutional network,
whose convolutional layers use zero-padding
\citep[see][]{goodfellow2016deep}.
%
Each layer but the last consists of a convolution
followed by an activation function that is applied
componentwise.  The
activations are $1$-Lipschitz and nonexpansive (examples include ReLU and tanh).
The kernels of the convolutional layers are
$K^{(i)} \in \mathbb{R}^{k \times k \times c \times c}$
for $i \in \lbrace 1,\dotsc, L \rbrace$.  Let $K = (K^{(1)},...,K^{(L)})$ be the
$L \times k \times k \times c \times c$-tensor obtained by concatening the
kernels for the various layers.  Vector $w$ represents the last layer;
the weights in the last layer are fixed with $|| w || = 1$.  
Let $W = L k^2 c^2$ be the total number of trainable parameters in the network.

We let $K_0^{(1)}, \dotsc, K_0^{(L)}$ take arbitrary fixed values
(interpreted as the initial values of the kernels) subject to the
constraint that, for all layers $i$, $|| \op(K_0^{(i)}) ||_2 = 1$.
(This is often the goal of initialization schemes.)  Let $K_0$ be the
corresponding $L \times k \times k \times c \times c$ tensor. We
provide a generalization bound in terms of distance from initialization,
along with other natural
parameters of the problem.  The distance is measured with
$|| K - K_0 ||_{\sigma} \eqdef \sum_{i=1}^L || \op(K^{(i)}) - \op(K_0^{(i)}) ||_2$.

For $\beta > 0$, define $\cK_{\beta}$ to be the
set of kernel tensors within $|| \cdot ||_{\sigma}$ distance $\beta$ of $K_0$,
and define $F_{\beta}$ to be set of functions computed by
CNNs with kernels in $\cK_{\beta}$.
That is,
$F_{\beta} = \lbrace f_K: || K - K_0 ||_{\sigma} \leq \beta \rbrace$.

Let $\ell: \R \times \R \rightarrow [0,1]$ be a loss function
such that $\ell(\cdot, y)$ is $\lambda$-Lipschitz for all $y$.  An example is
the $1/\lambda$-margin loss.

For a function $f$ from $\R^{d \times d \times c}$ to 
$\R$, let
$\ell_f(x,y) = \ell(f(x), y)$.  

We will use $S$ to denote a set 
$\lbrace (x_1,y_1), \dotsc, (x_m,y_m) \rbrace 
  = \lbrace z_1, \dotsc z_m \rbrace$
of random training examples where each $z_t = (x_t,y_t)$.

\begin{theorem}[Basic bounds] \label{thm:generalization}
For any $\eta > 0$, there is a $C > 0$ such that
for any $\beta, \delta > 0$, $\lambda \geq 1$, 
for any joint probability distribution $P$
over $\R^{d \times d \times c} \times \R$, if
a training set $S$ of $n$ examples is drawn
independently at random from $P$, then, with
probability at least $1-\delta$, for all $f \in F_{\beta}$, 
\begin{align*}
&     \mathE_{z \sim P} [\ell_f(z)] \leq 
    (1 + \eta) \mathE_S [\ell_f] 
  + \frac{C (W (\beta + \log (\lambda  n)) + \log(1/\delta))}{n}
\end{align*}
and, if $\beta \geq 5$, then
\begin{align*}
     \mathE_{z \sim P} [\ell_f(z)] \leq 
    \mathE_S [\ell_f] +
  C \sqrt{\frac{W (\beta + \log(\lambda)) + \log(1/\delta)}{n}}
\end{align*}
and otherwise
\begin{align*}
     \mathE_{z \sim P} [\ell_f(z)] \leq 
    \mathE_S [\ell_f] +
  C \left(
      \beta \lambda \sqrt{\frac{W}{n}}
       + \sqrt{\frac{\log(1/\delta)}{n}}
     \right).
\end{align*}
\end{theorem}
If Theorem~\ref{thm:generalization} is applied with the margin loss,
then $\mathE_{z \sim P} [\ell_f(z)]$ is in turn
an upper bound on the probability of misclassification on test
data.  Using the algorithm from \citep{sedghi2018singular},
$|| \cdot ||_{\sigma}$ may be efficiently computed.
Since $|| K - K_0 ||_{\sigma} \leq || \vecrm(K) - \vecrm(K_0)
||_1$ \citep{sedghi2018singular}, Theorem~\ref{thm:generalization}
yields the same bounds as a corollary if the definition of $F_{\beta}$
is replaced with the analogous definition using $|| \vecrm(K) -
\vecrm(K_0) ||_1$.


\subsection{Tools}

\begin{definition}
For $d \in N$, a set $G$ of real-valued functions with a common domain 
$Z$, we say that
$G$ is {\em $(B,d)$-Lipschitz parameterized}
if there is a norm $|| \cdot ||$ on $\R^d$ and
a mapping $\phi$ from the unit ball w.r.t.\ $|| \cdot ||$
in $\R^d$ to $G$ such that, 
for all $\theta$ and $\theta'$ such that
$|| \theta || \leq 1$ and $|| \theta' || \leq 1$, and
all $z \in Z$,
$
| (\phi(\theta))(z) - (\phi(\theta'))(z) | 
   \leq B || \theta - \theta' ||.
$
\end{definition}

The following lemma is essentially known.
Its proof, which uses standard techniques
\citep[see][]{Pol84,Tal94,talagrand1996new,barron1999risk,van2000empirical,gine2001consistency,mohri2018foundations},
is in Appendix~\ref{a:lipschitz}.
\begin{lemma}
\label{l:lipschitz}
Suppose a set $G$ of functions from a common domain $Z$
to $[0,M]$ is $(B,d)$-Lipschitz parameterized
for $B > 0$ and $d \in \N$.

Then, for any $\eta > 0$, there is a $C$ such that,
for all large enough $n \in \N$, for any $\delta > 0$, 
for any probability distribution $P$
over $Z$, if $S$ is obtained by sampling
$n$ times independently from $P$, then, with
probability at least $1-\delta$, for all $g \in G$,
\begin{align*}
     \mathE_{z \sim P} [g(z)] \leq 
    (1 + \eta) \mathE_S [g] + \frac{CM (d \log(B n) + \log(1/\delta))}{n}
\end{align*}
and if $B \geq 5$,
\begin{align*}
     \mathE_{z \sim P} [g(z)] \leq 
    \mathE_S [g] +
 C M \sqrt{\frac{d \log B + \log \frac{1}{\delta}}{n}},
\end{align*}
and, for all $B$, 
\begin{align*}
     \mathE_{z \sim P} [g(z)] \leq 
    \mathE_S [g] +
 C \left( B \sqrt{\frac{d}{n}} 
          + M \sqrt{\frac{\log \frac{1}{\delta}}{n}} \right).
\end{align*}
\end{lemma}

\subsection{Proof of Theorem~\ref{thm:generalization}}

\begin{definition}
Let $\ell_F = \{ \ell_f : f \in F \}$.
\end{definition}

We will prove Theorem~\ref{thm:generalization} 
by showing that $\ell_{F_{\beta}}$ is 
$\left(\beta \lambda e^{\beta}, W\right)$-Lipschitz parameterized.  
This will be achieved through
a series of lemmas.


%

%

\begin{lemma}\label{lemma:pro_single}
Choose $K \in \cK_{\beta}$ and a layer $j$.  Suppose
$\tK \in \cK_{\beta}$ satisfies $K^{(i)} = \tK^{(i)}$ for all $i \neq j$.  
Then, for all examples $(x,y)$,
$
| \ell(f_K(x),y) - \ell(f_{\tK} (x),y) |
 \leq \lambda e^{\beta}
   \norm{\op(K^{(j)}) - \op(\tK^{(j)})}_2.
$
\end{lemma}

\begin{proof}
For each layer $i$, let 
$\beta_i = || \op(K^{(i)}) - \op(K_0^{(i)}) ||_2$.

Since $\ell$ is $\lambda$-Lipschitz w.r.t. its first argument, we have that
$
| \ell(f_K(x),y) - \ell(f_{\tK} (x),y) |
 \leq \lambda | f_K(x) - f_{\tK} (x) |,
$
so it suffices to bound $| f_K(x) - f_{\tK} (x) |$.
Let $g_{\text{up}}$ be the function from the inputs to the whole
network with parameters $K$ to the inputs to the convolution in layer
$j$, and let $g_{\text{down}}$ be the function from the output of this
convolution to the output of the whole network, so that $f_K =
g_{\text{down}} \circ f_{\op(K^{(j)})} \circ g_{\text{up}}$.  Choose an
input $x$ to the network, and let $u = g_{\text{up}}(x)$.  Recalling
that $|| x || \leq 1$, and the non-linearities are nonexpansive, we have
$
\Vert u \Vert 
 \leq {\displaystyle \prod_{i < j} \norm{\op(K^{(i)})}_2 }.
$
Since the non-linearities are 1-Lipschitz,
and, recalling that $K^{(i)} = \tK^{(i)}$ for $i \neq j$, 
we have
\begin{align*}
| f_K(x) - f_{\tK}(x) | 
& = | g_{\text{down}}(\op(K^{(j)}) u) - g_{\text{down}}(\op(\tK^{(j)}) u) | \\
& \leq \left( \prod_{i > j} \norm{\op(K^{(i)})}_2 \right) \norm{\op(K^{(j)}) u - \op(\tildeK^{(j)}) u} \\
& \leq \left( \prod_{i > j} \norm{\op(K^{(i)})}_2 \right) \norm{\op(K^{(j)}) - \op(\tildeK^{(j)})}_2 \norm{u} \\
& \leq \left( \prod_{i \neq j} \norm{\op(K^{(i)})}_2 \right)
          \norm{\op(K^{(j)}) - \op(\tildeK^{(j)})}_2 \\
& \leq \left( \prod_{i \neq j} (1 + \beta_i) \right)
          \norm{\op(K^{(j)}) - \op(\tildeK^{(j)})}_2 \\
\end{align*}
where the last inequality uses the fact that 
$|| \op (K^{(i)}) - \op(K_0^{(i)}) ||_2 \leq \beta_i$ for all $i$ and
$|| \op (K_0^{(i)}) ||_2 = 1$ for all $i$.  

Now $\prod_{i \neq j} (1 + \beta_i) \leq \prod_{i = 1}^L (1 +
\beta_i)$, and the latter is maximized over the nonnegative
$\beta_i$'s subject to $\sum_{i \neq j} \beta_i \leq \beta$ when each
of them is $\beta/L$.  Since $(1 + \beta/L)^L \leq e^{\beta}$, this
completes the proof.
\end{proof}

Now we prove a bound when all the layers can change between $K$ and $\tK$.

\begin{lemma} \label{lemma_pro_all}
For any $K, \tK \in \cK_{\beta}$, for any input $x$ to the network,
$
| \ell(f_K(x),y) - \ell(f_{\tK} (x),y) |
 \leq \lambda e^{\beta} \norm{K - \tK}_{\sigma}.
$
\end{lemma}
\begin{proof}
Consider transforming $K$ to $\tK$ by replacing one layer of $K$ at a time with the corresponding layer in
$\tK$.  Applying Lemma~\ref{lemma:pro_single} to bound the distance
traversed with each replacement and combining this with 
the triangle inequality gives
\begin{align*}
| \ell(f_K(x),y) - \ell(f_{\tK} (x),y) |
& \leq 
\lambda e^{\beta}
 \sum_{j=1}^L \norm{\op(K^{(j)}) - \op(\tK^{(j)})}_2 
  = \lambda e^{\beta} \norm{K - \tK}_{\sigma}.
\end{align*}
\end{proof}

Now we are ready to prove our basic bound.

\begin{proof}[Proof (of Theorem~\protect\ref{thm:generalization})]
Consider the mapping $\phi$ from the 
ball w.r.t.\ $\norm{ \cdot }_{\sigma}$ of radius $1$ in 
$\R^{L k^2 c^2}$
centered at $\VEC(K_0)$
to $\ell_{F_{\beta}}$ defined by 
$\phi(\theta) = \ell_{f_{K_0 + \beta \VEC^{-1}(\theta)}}$, where
$\VEC^{-1}(\theta)$ is the reshaping of $\theta$ into
a $L \times k \times k \times c \times c$-tensor.  
Lemma~\ref{lemma_pro_all} implies that this
mapping is $\beta \lambda e^{\beta}$-Lipschitz.
Applying Lemma~\ref{l:lipschitz} completes the proof.
\end{proof}

\subsection{Comparisons}
\label{s:comparison.conv}

Since a convolutional network has an alternative parameterization as a
fully connected network, the bounds of \citep{bartlett2017spectrally}
have consequences for convolutional networks.  To compare our bound
with this, first, note that Theorem~\ref{thm:generalization},
together with standard model selection techniques, yields a
\begin{equation}
\label{e:nonuniform}
O \left( \sqrt{\frac{W \left(|| K - K_0 ||_{\sigma} + \log(\lambda)\right) 
                   + \log(1/\delta)}{n}}
\right)
\end{equation}
bound on $\mathE_{z \sim P} [\ell_f(z)] - \mathE_S [\ell_f(z)]$
(For more details, please see Appendix~\ref{a:nonuniform}.)
Translating the bound of
\citep{bartlett2017spectrally}
to our setting and notation directly yields a bound on
$\mathE_{z \sim P} [\ell_f(z)] - \mathE_S [\ell_f(z)]$ whose main
terms are proportional to
\begin{equation}
\label{e:BFT.regret.general}
  \frac{
       \lambda
        \left( \prod_{i=1}^L || \op(K^{(i)}) ||_2 \right)
        \left( 
          \sum_{i=1}^L 
             \frac{ || \op(K^{(i)})^{\top} - \op(K_0^{(i)})^{\top} ||_{2,1}^{2/3} }
                  {|| \op(K^{(i)}) ||_2^{2/3}}
        \right)^{3/2}
        \log( d^4 c^2 L)
        + \sqrt{\log(1/\delta)}
      }
      {
      \sqrt{n}
      }
\end{equation}
where, for a $p \times q$ matrix $A$,
$
|| A ||_{2,1} = || ( || A_{:,1} ||_2,..., || A_{:,q} ||_2 ) ||_1.
$
One can get an idea of how this bound relates
to (\ref{e:nonuniform}) by comparing the bounds in
a simple concrete case.  
Suppose that each of the convolutional
layers of the network parameterized by $K_0$ 
computes the identity function, and that 
$K$ is obtained from $K_0$ by adding $\epsilon$ to each
entry.  In this case, disregarding edge effects, for all $i$, 
$|| \op(K^{(i)}) ||_2 = 1 + \epsilon k^2 c$
and $|| K - K_0 ||_{\sigma} = \epsilon k^2 c L$
(as proved in Appendix~\ref{a:op}).
Also,
$|| \op(K^{(i)})^{\top} - \op(K_0^{(i)})^{\top} ||_{2,1}
 = (c d^2) (\epsilon \sqrt{c k^2}) = \epsilon c^{3/2} d^{2} k.$
We get additional simplification if we set
$\epsilon = \frac{1}{k^2}$.
In this case, 
(\ref{e:BFT.regret.general}) gives a constant times
\[
  \frac{
        (c+1)^{L} \sqrt{c} d (d/k)
        L^{3/2} 
        \lambda
        \log( d c L)
        + \sqrt{\log(1/\delta)}
      }
      {
      \sqrt{n}
      }
\]
where (\ref{e:nonuniform}) gives a constant times
\[
\frac{c^{3/2} k L 
      + c k \sqrt{\log(\lambda)} + \sqrt{\log(1/\delta)}}{\sqrt{n}}.
\]
In this scenario, the new bound is independent of $d$, and grows more
slowly with $\lambda$, $c$ and $L$.  
Note that $k \leq d$ (and, typically, it is much less).

This specific case illustrates a more general 
phenomenon
that holds
when the initialization is close to the identity, and
changes to the parameters are on a similar scale.

\citet{golowich2017size} established bounds that improve on the bounds
of \citep{bartlett2017spectrally} in some cases.  Their bound requires a
restriction on the activation function (albeit one that is satisfied
by the ReLU).  For large $n$, the 
main term of the natural consequence of Corollary 1
of their paper in the setting of this section, with the required additional
assumption on the activation function, grows like
\[
\frac{\lambda \sqrt{L} \prod_{i=1}^L || \op(K^{(i)}) ||_F}{\sqrt{n}}
 \approx \frac{\lambda (c d)^L \sqrt{L}}{\sqrt{n}},
\]
when $\epsilon = 1/k^2$.  (We note, however, that, in addition to not
trying to take account of the convolutional structure,
\citet{golowich2017size} also did not make an effort to obtain
stronger bounds in the case that the distance from initialization is
small.  On the other hand, we suspect that modifying their
proof analogously to \citep{bartlett2017spectrally} to do so
would not remove the exponential dependence on $L$.)


\section{A more general bound}

In this section, we generalize 
Theorem~\ref{thm:generalization}.  

\subsection{The setting}
\label{s:general.setting}

The more general setting 
concerns a neural network where the input is a
$d \times d \times c$ tensor whose flattening
has Euclidian norm at most $\chi$, and network's output
is a $m$-dimensional vector, which may be
logits for predicting a one-hot encoding of an
$m$-class classification problem.

The network is comprised of $L_c$ convolutional layers
followed by $L_f$ fully connected layers.
The $i$th convolutional
layer includes a convolution, with kernel
$K^{(i)} 
  \in \R^{k_{i} \times k_{i} \times c_{i - 1} \times c_{i}}$,
followed by a componentwise non-linearity and an
optional pooling operation.  We assume
that the non-linearity and any pooling 
operations are $1$-Lipschitz and nonexpansive.
Let $V^{(i)}$ be the matrix of weights for the $i$th fully
connected layer.  Let 
$\Theta = (K^{(1)},...,K^{(L_c)}, V^{(1)}, ..., V^{(L_f)})$ 
be all of the parameters of the network.
Let $L = L_c + L_f$.

We assume that, for all $y$, $\ell(\cdot, y)$ is $\lambda$-Lipschitz
for all $y$ and that $\ell(\hy, y) \in [0,M]$ for all $\hy$ and $y$.

An example $(x,y)$ includes a 
$d \times d \times c$-tensor $x$ and 
$y \in \R^m$.

We let $K_0^{(1)}, \dotsc, K_0^{(L_c)},V_0^{(1)},..., V_0^{(L_f)}$ 
take arbitrary fixed values subject to the constraint
that, for all convolutional layers $i$, 
$|| \op(K_0^{(i)}) ||_2 \leq 1 + \nu$, and
for all fully connected layers $i$, 
$|| V_0^{(i)} ||_2 \leq 1 + \nu$.
Let $\Theta_0 = (K_0^{(1)},...,K_0^{(L_c)}, V_0^{(1)},...,V_0^{(L_f)})$. 

For $\Theta = (K^{(1)},...,K^{(L_c)}, V^{(1)},...,V^{(L_f)})$
and $\tTheta = (\tK^{(1)},...,\tK^{(L_c)}, \tV^{(1)},...,\tV^{(L_f)})$.
define
\[
|| \Theta - \tTheta ||_{N}
 = 
   \left( \sum_{i = 1}^{L_c} || \op(K^{(i)}) - \op(\tK^{(i)}) ||_2 \right)
   + \sum_{i = 1}^{L_f} || V^{(i)} - \tV^{(i)} ||_2.
\]

For $\beta,\nu \geq 0$, 
define $\cF_{\beta, \nu}$ to be set of functions computed by CNNs 
as described in 
this subsection with
parameters within $|| \cdot ||_{N}$-distance $\beta$ of $\Theta_0$.
Let $\calO_{\beta,\nu}$ be the set of their parameterizations.

\begin{theorem}[General Bound] 
\label{thm:generalization.general}
For any $\eta > 0$, there is a constant $C$ such that
the following holds.
For any $\beta, \nu, \chi > 0$,
for any $\delta > 0$, 
for any joint probability distribution $P$
over $\R^{d \times d \times c} \times \R^m$
such that, with probability 1, $(x,y) \sim P$ satisfies
$|| \VEC(x) ||_2 \leq \chi$,
under the assumptions of this section, if
a training set $S$ of $n$ examples is drawn
independently at random from $P$, then, with
probability at least $1-\delta$, for all $f \in \cF_{\beta, \nu}$,
\begin{align*}
&     \mathE_{z \sim P} [\ell_f(z)] 
  \leq
    (1 + \eta) \mathE_S [\ell_f]  
  +
  \frac{C
        M
          \left(W 
       \left(
              \beta + \nu L +
              \log(\chi \lambda \beta n) 
         \right)
          + \log(1/\delta)\right)}{n}
\end{align*}
and, if $\chi \lambda \beta (1 + \nu + \beta/L)^L \geq 5$, 
\begin{align*}
 &     \mathE_{z \sim P} [\ell_f(z)] 
%
 \leq
 \mathE_{S} [\ell_f] 
  + C M \sqrt{\frac{W 
              \left(\beta + \nu L +
              \log(\chi \lambda \beta )\right)
                        + \log(1/\delta)}{n}}
\end{align*}
and a bound of
\begin{align*}
 &     \mathE_{z \sim P} [\ell_f(z)] 
 \leq
     \mathE_S [g] +
  C \left( \chi \lambda \beta (1 + \nu + \beta/L)^L
           \sqrt{\frac{W}{n}} 
           + M \sqrt{\frac{\log \frac{1}{\delta}}{n}} \right)
\end{align*}
holds for all $\chi, \lambda, \beta > 0$.
\end{theorem}


\subsection{Proof of Theorem~\ref{thm:generalization.general}}

We will prove Theorem~\ref{thm:generalization.general} 
by using $|| \cdot ||_N$ to witness the fact that
$\ell_{\cF_{\beta,\nu}}$ is 
$\left(\chi \lambda \beta (1 + \nu + \beta/L)^L,W\right)$-Lipschitz
parameterized.

The first two lemmas concern the effect of changing a 
single layer.  Their proofs are very similar to
the proof of Lemma~\ref{lemma:pro_single}, and are in the
Appendices~\ref{a:pro_single.conv} and \ref{a:pro_single.full}.
\begin{lemma}\label{lemma:pro_single.conv}
Choose 
$\Theta = (K^{(1)},...,K^{(L_c)}, V^{(1)},...,V^{(L_f)}),
\tTheta = (\tK^{(1)},...,\tK^{(L_c)}, \tV^{(1)},...,\tV^{(L_f)})
\in \calO_{\beta,\nu}$
and a convolutional layer $j$.
Suppose that $K^{(i)} = \tK^{(i)}$ for all 
convolutional layers $i \neq j$ and $V^{(i)} = \tV^{(i)}$ for all 
fully connected layers $i$.
Then, for all examples $(x,y)$,
\[
| \ell(f_{\Theta}(x),y) - \ell(f_{\tTheta} (x),y) |
 \leq \chi \lambda 
        (1 + \nu + \beta/L)^L
   \norm{\op(K^{(j)}) - \op(\tK^{(j)})}_2.
\]
\end{lemma}
\begin{lemma}\label{lemma:pro_single.full}
Choose 
$\Theta = (K^{(1)},...,K^{(L_c)}, V^{(1)},...,V^{(L_f)}),
\tTheta = (\tK^{(1)},...,\tK^{(L_c)}, \tV^{(1)},...,\tV^{(L_f)})
\in \calO_{\beta,\nu}$
and a fully connected layer $j$.
Suppose that $K^{(i)} = \tK^{(i)}$ for all 
convolutional layers $i$ and $V^{(i)} = \tV^{(i)}$ for all 
fully connected layers $i \neq j$.
Then, for all examples $(x,y)$,
\[
| \ell(f_{\Theta}(x),y) - \ell(f_{\tTheta} (x),y) |
 \leq \chi \lambda 
   \left( 1 + \nu + \beta/L \right)^L  
   \norm{V^{(j)} - \tV^{(j)}}_2.
\]
\end{lemma}

Now we prove a bound when all the layers can change between 
$\Theta$ and $\tTheta$.

\begin{lemma} \label{lemma_pro_all.general}
For any 
$\Theta = (K^{(1)},...,K^{(L_c)}, V^{(1)},...,V^{(L_f)}),
\tTheta = (\tK^{(1)},...,\tK^{(L_c)}, \tV^{(1)},...,\tV^{(L_f)}) \in \calO_{\beta,\nu}$, 
for any input $x$,
\begin{align*}
| \ell(f_{\Theta}(x),y) - \ell(f_{\tTheta} (x),y) |
& \leq \chi \lambda 
      \left( 1 + \nu + \beta/L \right)^L  
      \norm{\Theta - \tTheta}_N.  
\end{align*}
\end{lemma}
\begin{proof}
Consider transforming $\Theta$ to $\tTheta$ by replacing one layer 
at a time of $\Theta$
with the corresponding layer in $\tTheta$.  
Applying Lemma~\ref{lemma:pro_single.conv} to bound the distance
traversed with each replacement of a convolutional layer,
and Lemma~\ref{lemma:pro_single.full} to bound the distance
traversed with each replacement of a fully connected layer,
and combining this with 
the triangle inequality gives
the lemma.
\end{proof}

Now we are ready to prove our more general bound.

\begin{proof}[Proof (of Theorem~\protect\ref{thm:generalization.general})]
Consider the mapping $\phi$ from the 
ball of $|| \cdot ||_N$-radius $1$ centered at $\Theta_0$
to $\ell_{\cF_{\beta,\nu}}$ defined by 
$\phi(\Theta) = \ell_{f_{\Theta_0 + \beta \Theta}}$.  
Lemma~\ref{lemma_pro_all} implies that this
mapping is 
$\left(\chi \lambda \beta 
      \left( 1 + \nu + \beta/L \right)^L,W\right)$-Lipschitz.  
Applying Lemma~\ref{l:lipschitz} completes the proof.
\end{proof}

  \begin{figure}[t]
  \centering
  \includegraphics[width=8.5cm]{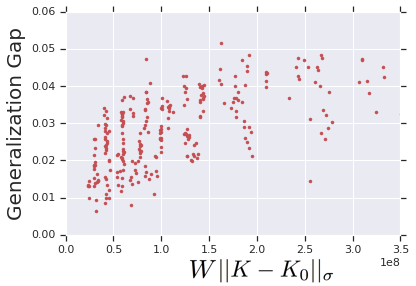}
  \caption{\label{fig:error_vs_wb} Generalization gaps for a 10-layer all-conv model on CIFAR10 dataset.}
  \end{figure}

 \begin{figure}[t]
  \centering
  \includegraphics[width=8.5cm]{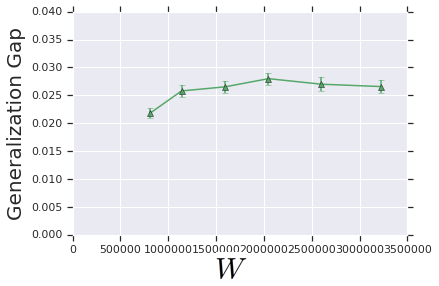}
  \caption{\label{fig:gen_error_vs_w} Generalization gap as a function of  $W$}
  \end{figure}

\subsection{More comparisons}

Theorem~\ref{thm:generalization.general} applies in the case that there
are no convolutional layers, i.e.\ for a fully connected network.  
In this subsection, we compare its
bound in this case with the bound of \citep{bartlett2017spectrally}.
Because the bounds are in terms of different quantities, we compare
them in a simple concrete case.  In this case,
for $D = c d^2$, each hidden layer has $D$ components,
and there are $D$ classes.  For all $i$,
$V_0^{(i)} = I$ and $V^{(i)} = I + H/\sqrt{D}$, where $H$ is a Hadamard matrix
(using the Sylvester construction), and $\chi = M = 1$.
Then, dropping the
superscripts, each layer $V$ has 
$
|| V ||_2  = 2,\; || V - V_0 ||_2  = 1,\; || V - V_0 ||_{2,1}  = D. \\
$

Further, in the notation of Theorem~\ref{thm:generalization.general}, 
$W = D^2 L$, and $\beta = L$, and $\nu = 0$.  
Plugging into to Theorem~\ref{thm:generalization.general} yields a
bound on the generalization gap proportional to
\[
 \frac{D L
       + D \sqrt{L \log(\lambda)}
                        + \sqrt{\log(1/\delta)}}{\sqrt{n}}
\]
where, in this case, the bound of \citep{bartlett2017spectrally} 
is proportional to
\[
\frac{
       \lambda
        2^L
        L^{3/2}
        D \log( D L)
        + \sqrt{\log(1/\delta)}
      }
      {
      \sqrt{n}
      }
\]
and, when $D$ is large relative
to $L$, Corollary 1 of \citep{golowich2017size} 
(approximately) gives
\[
\frac{\lambda 2^{L/2} \sqrt{L} D^{L/2}+ \sqrt{\log(1/\delta)}}{\sqrt{n}}.
\]

\section{Experiments}
We trained a 10-layer all-convolutional model on the 
CIFAR-10 dataset.
The architecture was similar to VGG \citep{simonyan2014very}.  The
network was trained with dropout regularization and an exponential
learning rate schedule.  We define the generalization gap as the
difference between train error and test error.  In order to analyze
the effect of the number of network parameters on the generalization gap, we
scaled up the number of channels in each layer, while keeping other
elements of the architecture, including the depth, fixed.  Each
network was trained repeatedly, sweeping over different values of the
initial learning rate and batch sizes $32, 64, 128$. 
For each setting
the results were averaged over five different random initializations.
Figure~\ref{fig:error_vs_wb} shows the generalization gap for
different values of $W || K - K_0||_{\sigma}$.
As in the bound of
Theorem~\ref{thm:generalization.general}, the generalization
gap increases with $W || K - K_0 ||_{\sigma}$.  Figure~\ref{fig:gen_error_vs_w} shows
that as the network becomes more over-parametrized, the generalization
gap remains almost flat with increasing $W$. This is expected due to
role of over-parametrization on
generalization~\citep{neyshabur2019towards}. 
An explanation of this phenomenon that is consistent with the
bound presented here is that, ultimately,
increasing $W$ leads to a decrease in value of
$|| K - K_0 ||_{\sigma}$; see Figure~\ref{fig:B_vs_w}.
The fluctuations in Figure~\ref{fig:B_vs_w} are partly due to the fact that training  neural networks is not an stable process. 
We provide the medians $|| K - K_0 ||_{\sigma}$ for different values of $W$ in Figure~\ref{fig:beta_w_median}.

\begin{figure}%
    \centering
    \subfloat[ mean and error bar]{{\includegraphics[width=6.2cm]{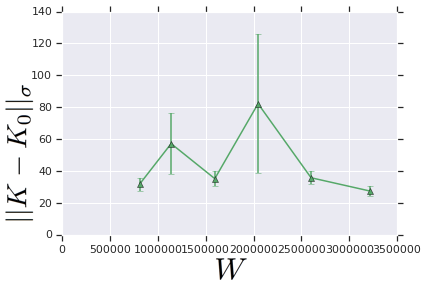} } \label{fig:B_vs_w}}%
    \qquad
        \subfloat[median]{{\includegraphics[width=6.2cm]{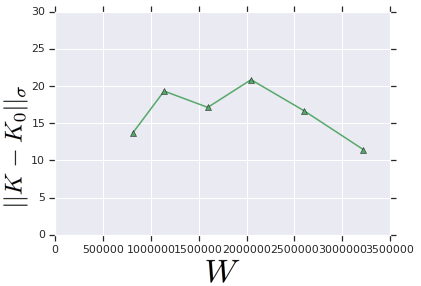} } \label{fig:beta_w_median}}%
    \caption{$|| K - K_0||_{\sigma}$ as a function of $W$.}
\end{figure}

%
%
%
%

\subsubsection*{Acknowledgments}

We thank Peter Bartlett, Jaeho Lee and Sam Schoenholz for valuable
conversations, and anonymous reviewers for helpful comments.

\bibliography{references}
\bibliographystyle{plainnat}

\appendix

\section{Proof of Lemma~\protect\ref{l:lipschitz}}
\label{a:lipschitz}

\begin{definition}
If $(X,\rho)$ is a metric space and $H \subseteq X$, we say that $G$ is an
$\epsilon$-cover of $H$ with respect to $\rho$ if every $h \in H$ has a $g \in G$
such that $\rho(g,h) \leq \epsilon$.  Then
$\cN_{\rho}(H,\epsilon)$ denotes the size of the smallest 
$\epsilon$-cover of $H$ w.r.t.\ $\rho$.
\end{definition}

\begin{definition}
For a domain $Z$, define a metric $\rhomax$ on pairs of functions 
from $Z$ to $\R$ by $\rhomax(f,g) = \sup_{z \in Z} | f(z) - g(z) |$.
\end{definition}


We need two lemmas in terms of these covering numbers.
The first is by now a standard bound from Vapnik-Chervonenkis theory
\citep{VC71,Vap82,Pol84}.  For example, it is
a direct consequence of \citep[Theorem 3]{Hau92}.
\begin{lemma}
\label{l:cover.relative}
For any $\eta > 0$, there is a constant $C$ depending
only on $\eta$ such that the following holds.
Let $G$ be an arbitrary set of functions from a common domain $Z$
to $[0,M]$.
If there are constants $B$ and $d$ such that,
$\cN_{\rhomax}(G,\epsilon) \leq \left( \frac{B}{\epsilon} \right)^d$
for all $\epsilon > 0$,
then,
for all large enough $n \in \N$, for any $\delta > 0$, 
for any probability distribution $P$
over $Z$, if $S$ is obtained by sampling
$n$ times independently from $P$, then, with
probability at least $1-\delta$, for all $g \in G$,
\begin{align*}
     \mathE_{z \sim P} [g(z)] \leq 
    (1 + \eta) \mathE_S[g] +
  \frac{CM (d \log(B n) + \log(1/\delta))}{n}.
\end{align*}
\end{lemma}

We will also use the following, which is the combination of 
(2.5) and (2.7) of
\citep{gine2001consistency}.
\begin{lemma}
\label{l:cover.regret.bigB}
Let $G$  
be an arbitrary set of functions from a common domain $Z$
to $[0,M]$.
If 
there are constants $B \geq 5$ and $d$ such that
$\cN_{\rhomax}(G,\epsilon) \leq \left( \frac{B}{\epsilon} \right)^d$
for all $\epsilon > 0$,
for all large enough $n \in \N$, for any $\delta > 0$, 
for any probability distribution $P$
over $Z$, if $S$ is obtained by sampling
$n$ times independently from $P$, then, with
probability at least $1-\delta$, for all $g \in G$,
\begin{align*}
     \mathE_{z \sim P} [g(z)] \leq 
    \mathE_S [g] +
 C M \sqrt{\frac{d \log B + \log \frac{1}{\delta}}{n}},
\end{align*}
where $C$ is an absolute constant.
\end{lemma}

The above bound only holds for $B \geq 5$.
The following, which can be obtained by combining Talagrand's
Lemma with the
standard bound on Rademacher complexity in terms of the 
Dudley entropy integral (see \citep{van2000empirical,Bar13}), yields a 
bound for all $B$.
\begin{lemma}
\label{l:cover.regret.smallB}
Let $G$  be an arbitrary set of functions from a common domain $Z$
to $[0,M]$.
If there are constants $B > 0$ and $d$ such that
$\cN_{\rhomax}(G,\epsilon) \leq \left( \frac{B}{\epsilon} \right)^d$
for all $\epsilon > 0$,
then, 
for all large enough $n \in \N$, for any $\delta > 0$, 
for any probability distribution $P$
over $Z$, if $S$ is obtained by sampling
$n$ times independently from $P$, then, with
probability at least $1-\delta$, for all $g \in G$,
\begin{align*}
     \mathE_{z \sim P} [g(z)] \leq 
    \mathE_S [g] +
 C \left( B \sqrt{\frac{d}{n}} 
          + M \sqrt{\frac{\log \frac{1}{\delta}}{n}} \right),
\end{align*}
where $C$ is an absolute constant.
\end{lemma}

So now we want a bound on $\cN_{\rhomax}(G,\epsilon)$ for
Lipschitz-parameterized classes.  For this, we need
the notion of a packing which we
now define.
\begin{definition}
For any metric space $(X,\rho)$ and any
$H \subseteq S$, let $\cM_{\rho}(H,\epsilon)$
be the size of the largest subset of $H$ 
whose members are pairwise at a distance
greater than $\epsilon$ w.r.t.\ $\rho$.
\end{definition}

\begin{lemma}[\citep{kolmogorov1959varepsilon}]
\label{l:covering_by_packing}
For any metric space $(X,\rho)$, any $H \subseteq X$,
and any $\epsilon > 0$, we have
\[
\cN_{\rho}(H, \epsilon) \leq \cM_{\rho}(H, \epsilon).
\]
\end{lemma}

We will also need a lemma about covering a ball by
smaller balls.  This is probably also already known,
and uses a standard proof 
\citep[see][Lemma 4.1]{Pol90}, but we haven't
found a reference for it.
\begin{lemma}
\label{l:ball_by_smaller_balls}
Let
\begin{itemize}
\item $d$ be a positive integer,
\item $|| \cdot ||$ be a norm
\item $\rho$ be the metric induced by $|| \cdot ||$, and
\item $\kappa, \epsilon > 0$.
\end{itemize}
A ball in $\R^d$ of radius $\kappa$ 
w.r.t.\ $\rho$ can be
covered by $\left( \frac{3 \kappa}{\epsilon} \right)^d$
balls of radius $\epsilon$.
\end{lemma}
\begin{proof}
We may assume without loss of generality that
$\kappa > \epsilon$.   Let $q > 0$ be the volume of the unit ball w.r.t.\ 
$\rho$ in $\R^{d}$.  
Then 
the volume of any $\alpha$-ball with respect
to $\rho$ is $\alpha^{d} q$.
Let $B$ be the ball of radius $r$ in $\R^d$.  
The $\epsilon/2$-balls centered at the members
of any $\epsilon$-packing of $B$
are disjoint.  Since these centers are contained in $B$,
the balls are contained in a ball of radius $\kappa + \epsilon/2$.
Thus
\[
\cM_{\rho}(B,\epsilon) \left(\frac{\epsilon}{2}\right)^{d} q
 \leq \left( \kappa + \frac{\epsilon}{2}\right)^{d} q
 \leq \left( \frac{3 \kappa}{2} \right)^{d} q.
\]
Solving for $\cM_{\rho}(B,\epsilon)$ 
and applying Lemma~\ref{l:covering_by_packing}
completes the proof.
\end{proof}

\medskip

We now prove 
Lemma~\ref{l:lipschitz}.  
Let $|| \cdot ||$
be the norm witnessing the fact that $G$ is $(B,d)$-Lipschitz
parameterized, and let $\cB$ be the unit ball in $\R^d$ w.r.t.\ 
$|| \cdot ||$ and let $\rho$ be the metric induced by
$|| \cdot ||$.  
Then, for any $\epsilon$, 
an $\epsilon/B$-cover of $\cB$ w.r.t.\ $\rho$ induces
an $\epsilon$-cover of $G$ w.r.t.\ $\rhomax$, so
\[
N_{\rhomax}(G,\epsilon) \leq N_{\rho}(\cB,\epsilon/B).
\]
Applying Lemma~\ref{l:ball_by_smaller_balls}, this
implies
\[
N_{\rhomax}(G,\epsilon) \leq \left( \frac{3 B}{\epsilon} \right)^d.
\]
Then applying Lemma~\ref{l:cover.relative},
Lemma~\ref{l:cover.regret.bigB} and Lemma~\ref{l:cover.regret.smallB}
completes the proof.

\section{Proof of (\protect\ref{e:nonuniform})}
\label{a:nonuniform}

For $\delta > 0$, and for each $j \in \N$, let $\beta_j = 5 \times 2^j$ 
let $\delta_j = \frac{1}{2 j^2}$.  Taking a union bound
over an application of Theorem~\ref{thm:generalization}
for each value of $j$, with probability at least
$1 - \sum_j \delta_j \geq 1 - \delta$, for all $j$, and
all $f \in F_{\beta_j}$
\begin{align*}
&     \mathE_{z \sim P} [\ell_f(z)] \leq 
    (1 + \eta) \mathE_S [\ell_f(z)] 
  + \frac{C (W (\beta_j + \log (\lambda  n)) + \log(j/\delta))}{n}
\end{align*}
and
\begin{align*}
     \mathE_{z \sim P} [\ell_f(z)] \leq 
    \mathE_S [\ell_f(z)] +
  C \sqrt{\frac{W (\beta_j + \log(\lambda)) + \log(j/\delta)}{n}}.
\end{align*}
For any $K$, if we apply these bounds in the case of
the least $j$ such that $|| K - K_0 ||_{\sigma} \leq \beta_j$, we get
\begin{align*}
&     \mathE_{z \sim P} [\ell_f(z)] \leq 
    (1 + \eta) \mathE_S [\ell_f(z)] 
  + \frac{C (W (2 || K - K_0 ||_{\sigma} + \log (\lambda  n)) + 
   \log(\log(|| K - K_0 ||_{\sigma})/\delta))}{n}
\end{align*}
and
\begin{align*}
     \mathE_{z \sim P} [\ell_f(z)] \leq 
    \mathE_S [\ell_f(z)] +
  C \sqrt{\frac{W (2 || K - K_0 ||_{\sigma} + \log(\lambda)) + \log(\log(|| K - K_0 ||_{\sigma})/\delta)}{n}},
\end{align*}
and simplifying completes the proof.

\section{The operator norm of $\op(K^{(i)})$}
\label{a:op}

Let $J = K^{(i)} - K_0^{(i)}$.  Since $|| \op(K^{(i)}) ||_2 = 1 + || \op(J) ||_2$,
it suffices to find $|| \op(J) ||_2$.

For the rest of this section, we number
indices from $0$, let $[d] = \{ 0,...,d-1\}$, and
define $\omega = \exp(2 \pi i/d)$.  To facilitate the
application of matrix notation, pad the $k \times k \times c \times c$
tensor $J$ out with zeros to make a $d \times d \times c \times c$
tensor $\tJ$.  

The following lemma is an immediate consequence of Theorem 6 
of \cite{sedghi2018singular}.  
\begin{lemma}[\cite{sedghi2018singular}]
\label{l:convsv}
Let $F$ be the complex $d \times d$ matrix defined by $F_{ij} = \omega^{ij}$.  

For each $u,v \in [d] \times [d]$, let $P^{(u,v)}$ be the
$c \times c$ matrix given by $P_{k \ell}^{(u,v)} = (F^T \tJ_{:,:,k,\ell} F)_{uv}$.
Then
\[
|| \op(J) ||_2 = \max_{u,v} || P^{(u,v)} ||_2.
\]
\end{lemma}

First, note that, by symmetry, for each $u$ and $v$, all components
of $P^{(u,v)}$ are the same.  Thus,
\begin{equation}
\label{e:uvnorm}
|| P^{(u,v)} ||_2 = c | P_{00}^{(u,v)} |.  
\end{equation}

For any $u,v$,
\[
| P_{00}^{(u,v)} | 
 = \left| \sum_{p, q} \omega^{u p} \omega^{v q} \tJ_{p,q,0,0}  \right|
 \leq \epsilon k^2
\]
and $P_{00}^{(0,0)} = \epsilon k^2$.  Combining this
with (\ref{e:uvnorm}) and Lemma~\ref{l:convsv},
$|| \op(J) ||_2 = \epsilon c k^2$, which implies
$|| \op(K) ||_2 = 1 + \epsilon c k^2$.

\section{Proof of Lemma~\protect\ref{lemma:pro_single.conv}}
\label{a:pro_single.conv}

For each convolutional layer $i$, let 
$\beta_i = || \op(K^{(i)}) - \op(K_0^{(i)}) ||_2$,
and, for each fully connected layer $i$, 
let $\gamma_i = || V^{(i)} - V_0^{(i)} ||_2$.

Since $\ell$ is $\lambda$-Lipschitz w.r.t. its first argument, we have that
$
| \ell(f_{\Theta}(x),y) - \ell(f_{\tTheta} (x),y) |
 \leq \lambda | f_{\Theta}(x) - f_{\tTheta} (x) |.
$
Let $g_{\text{up}}$ be the function from the inputs to the whole network with parameters $\Theta$  to the inputs to the convolution in
layer $j$, and let $g_{\text{down}}$ be the function from the output of this convolution to the output of
the whole network, so that $f_{\Theta} = g_{\text{down}} \circ f_{\op(K^{(j)})} \circ g_{\text{up}}$. 
Choose an input $x$ to the network, and let $u = g_{\text{up}}(x)$. 
Recalling that $|| x || \leq \chi$, and that the non-linearities 
and pooling operations are 
non-expansive, we have
$
\Vert u \Vert 
 \leq \chi {\displaystyle \prod_{i < j} \norm{\op(K^{(i)})}_2 }.
$
Using the fact that 
the non-linearities are $1$-Lipschitz, we have
\begin{align*}
| f_{\Theta}(x) - f_{\tTheta}(x) | 
& = | g_{\text{down}}(\op(K^{(j)}) u) - g_{\text{down}}(\op(\tK^{(j)}) u) | \\
& \leq \left( \prod_{i > j} \norm{\op(K^{(i)})}_2 \right) \left( \prod_i \norm{V^{(i)}}_2 \right)  \norm{\op(K^{(j)}) u - \op(\tildeK^{(j)}) u}  \\
& \leq \left( \prod_{i > j} \norm{\op(K^{(i)})}_2 \right) \left( \prod_i \norm{V^{(i)}}_2 \right) \norm{\op(K^{(j)}) - \op(\tildeK^{(j)})}_2 \norm{u} \\
& \leq \chi 
       \left( \prod_{i \neq j} \norm{\op(K^{(i)})}_2 \right) \left( \prod_i \norm{V^{(i)}}_2 \right)
          \norm{\op(K^{(j)}) - \op(\tildeK^{(j)})}_2 \\
& \leq \chi \left( \prod_{i \neq j} (1 + \nu + \beta_i) \right)
          \left( \prod_{i} (1 + \nu + \gamma_i) \right)
          \norm{\op(K^{(j)}) - \op(\tildeK^{(j)})}_2 \\
\end{align*}
where the last inequality uses the fact that 
$|| \op (K^{(i)}) - \op (K_0^{(i)}) ||_2 \leq \beta_i$ for all $i$,
$|| V^{(i)} - V_0^{(i)} ||_2 \leq \gamma_i$ for all $i$, 
$|| \op (K_0^{(i)}) ||_2 \leq 1 + \nu$ for all $i$ and
$|| V_0^{(i)} ||_2 \leq 1 + \nu$ for all $i$.

Since $\left( \prod_{i \neq j} (1 + \nu + \beta_i) \right) \left( \prod_i (1 + \nu + \gamma_i) \right)
 \leq \left( \prod_i (1 + \nu + \beta_i) \right) \left( \prod_i (1 + \nu + \gamma_i) \right)$, and
the latter is maximized subject to $\left( \sum_i \beta_i \right) + \sum_i \gamma_i \leq \beta$
when each summand is $\beta/L$, this completes the proof.

\section{Proof of Lemma~\protect\ref{lemma:pro_single.full}}
\label{a:pro_single.full}

For each convolutional layer $i$, let 
$\beta_i = || \op(K^{(i)}) - \op(K_0^{(i)}) ||_2$,
and, for each fully connected layer $i$, 
let $\gamma_i = || V^{(i)} - V_0^{(i)} ||_2$.

Since $\ell$ is $\lambda$-Lipschitz w.r.t. its first argument, we have that
$
| \ell(f_{\Theta}(x),y) - \ell(f_{\tTheta} (x),y) |
 \leq \lambda | f_{\Theta}(x) - f_{\tTheta} (x) |.
$
Let $g_{\text{up}}$ be the function from the inputs to the whole network with parameters $\Theta$  to the inputs to 
fully connected layer layer $j$, and let $g_{\text{down}}$ be the function from the output of this layer to the output of
the whole network, so that $f_{\Theta} = g_{\text{down}} \circ f_{V^{(j)}}\circ g_{\text{up}}$. 
Choose an input $x$ to the network, and let $u = g_{\text{up}}(x)$. 
Recalling that $|| x || \leq \chi$, and that the non-linearities 
and pooling operations are non-expansive, we have
$
\Vert u \Vert 
 \leq \chi \left( {\displaystyle \prod_{i} \norm{\op(K^{(i)})}_2 } \right)
           \left( {\displaystyle \prod_{i < j} \norm{V^{(i)}}_2 } \right).
$
Thus
\begin{align*}
| f_{\Theta}(x) - f_{\tTheta}(x) | 
& = | g_{\text{down}}(V^{(j)} u) - g_{\text{down}}(\tV^{(j)} u) | \\
& \leq \left( \prod_{i>j} \norm{V^{(i)}}_2 \right) \norm{V^{(j)} u - \tV^{(j)} u} \\
& \leq \left( \prod_{i>j} \norm{V^{(i)}}_2 \right) \norm{V^{(j)} - \tV^{(j)}}_2 \norm{u} \\
& \leq \chi 
       \left( \prod_i \norm{\op(K^{(i)})}_2 \right) \left( \prod_{i \neq j} \norm{V^{(i)}}_2 \right)
          \norm{V^{(j)} - \tV^{(j)}}_2
           \\
& \leq \chi \left( \prod_{i} (1 + \nu + \beta_i) \right)
          \left( \prod_{i \neq j} (1 + \nu + \gamma_i) \right)
           \norm{V^{(j)} - \tV^{(j)}}_2.
\end{align*}

Since $\left( \prod_{i} (1 + \nu + \beta_i) \right) \left( \prod_{i \neq j} (1 + \nu + \gamma_i) \right)
 \leq \left( \prod_i (1 + \nu + \beta_i) \right) \left( \prod_i (1 + \nu + \gamma_i) \right)$, and
the latter is maximized subject to $\left( \sum_i \beta_i \right) + \sum_i \gamma_i \leq \beta$
when each summand is $\beta/L$, this completes the proof.

\end{document}